\documentclass[letterpaper, 10 pt, conference]{ieeeconf}
\IEEEoverridecommandlockouts    
\usepackage{times}

\usepackage{multicol}
\usepackage[bookmarks=true,colorlinks=false,hidelinks]{hyperref}
\usepackage{soul}
\usepackage{amsfonts}       
\usepackage{amsmath}
\usepackage{amssymb}
\usepackage{graphicx}
\usepackage{epstopdf}
\usepackage{mathtools}
\usepackage{dsfont}
\usepackage{tikz}
\usepackage{siunitx}
\usepackage{hyperref}
\usepackage[font=footnotesize]{caption}
\usepackage[noabbrev,capitalize]{cleveref}
\usepackage{balance}    
\usepackage{xcolor}
\usepackage{tcolorbox}
\usepackage{algorithm}
\usepackage{algpseudocode}
\usepackage{mathtools}
\usepackage{dsfont}
\usepackage{mathrsfs}
\usepackage{stmaryrd}
\usepackage{placeins}
\usepackage{upgreek}
\usepackage{wrapfig}
\usepackage{bm}
\usepackage{cases}
\usepackage{bbm}
\usepackage{noel}
\usepackage{subfloat}

\crefname{equation}{}{}
\title{\LARGE \bf Robust Agility via Learned Zero Dynamics Policies}

\author{Noel Csomay-Shanklin$^{1*}$, William D. Compton$^{1*}$, Ivan Dario Jimenez Rodriguez$^{1*}$, \\Eric R. Ambrose$^{2}$, Yisong Yue$^{1}$, Aaron D. Ames$^{1}$%
\thanks{$^*$denotes equal contribution. $^1$Authors are with the Department of Computing and Mathematical Sciences, California Institute of Technology,
Pasadena, CA 91125. $^2$Authors are with NASA Jet Propulsion Laboratory.}%
\thanks{This research was supported by Technology Innovation Institute (TII), NSF Graduate Research Fellowship No. DGE‐1745301, AeroVironment, NSF Grant No. 1918655 and Raytheon, Beyond Limits, JPL RTD 1643049.}
}

\begin{document}

\maketitle

\thispagestyle{empty}
\pagestyle{empty}

\begin{abstract}
    We study the design of robust and agile controllers for hybrid underactuated systems.
    Our approach breaks down the task of creating a stabilizing controller into: 1) learning a mapping that is invariant under optimal control, and 2) driving the actuated coordinates to the output of that mapping.
    This approach, termed Zero Dynamics Policies, exploits the structure of underactuation by restricting the inputs of the target mapping to the subset of degrees of freedom that cannot be directly actuated, thereby achieving significant dimension reduction.
    Furthermore, we retain the stability and constraint satisfaction of optimal control while reducing the online computational overhead.
    We prove that controllers of this type stabilize hybrid underactuated systems and experimentally validate our approach on the 3D hopping platform, ARCHER. 
    Over the course of 3000 hops the proposed framework demonstrates robust agility, maintaining stable hopping while rejecting disturbances on rough terrain.
\end{abstract}
\IEEEpeerreviewmaketitle



\section{Introduction}
The underactuated dynamics inherent to legged locomotion, swimming, and dexterous manipulation impose fundamental limits on controller performance and necessitate a critical understanding of the system's flow to achieve complex behaviors. 
Underactuation prevents arbitrarily shaping a system's dynamics, undermining the assumptions of many control-theoretic methods such as feedback linearization \cite{sastry_linearization_1999} and offline trajectory tracking. 
This work leverages recent advances in controller design for underactuated systems \cite{pmlr-v168-rodriguez22a, compton2024constructivenonlinearcontrolunderactuated}, optimal control \cite{liberzon_calculus_2012}, and their integration with computational learning methods to design feedback strategies that exploit the structure of underactuation, enabling the agile and robust behavior shown in \cref{fig:experiments}.


\begin{figure}[t!]
    \centering
    \href{https://vimeo.com/923800815?share=copy}{
    \includegraphics[width=0.985\columnwidth]{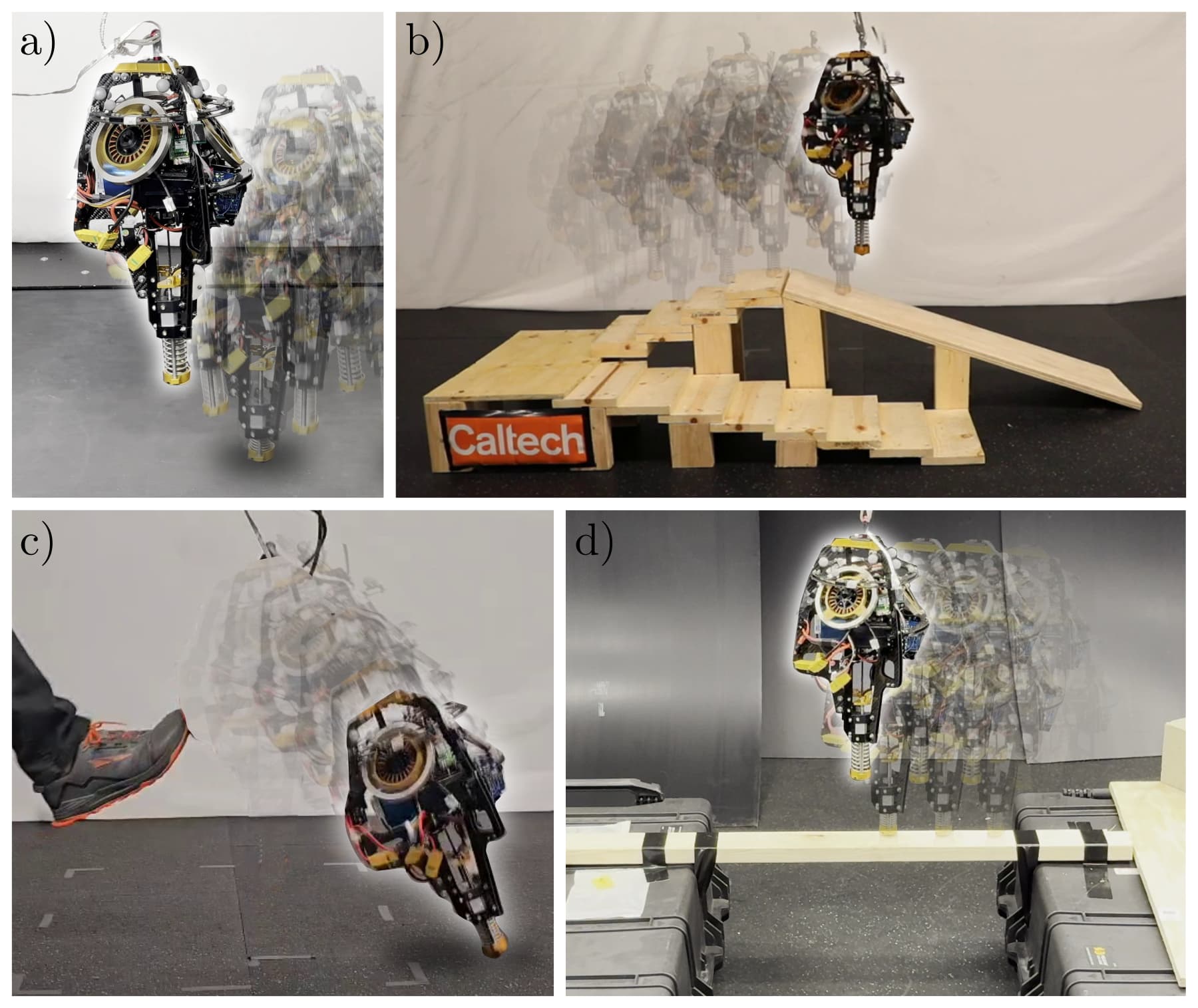}
    }
    \vspace{-5mm}
    \caption{Experiments run with Zero Dynamics Policies: a) treadmill hopping with disturbances up to 1 mile per hour, b) 1.5" stair climbing and 20° ramp descending, c) disturbance rejection, and d) hopping across a 2x4.}
    \vspace{-5.4mm}
    \label{fig:experiments}
\end{figure}

A predominant method for controlling underactuated systems is Model Predictive Control (MPC) \cite{borrelli2017predictive,mayne2000constrained}, which leverages concepts from optimal control over a prediction horizon to achieve stabilization \cite{wensing2022optimizationbasedcontroldynamiclegged}. Performance of MPC controllers improves with longer horizons and finer time discretizations, both of which conflict with its strict real-time computational requirements.
To address the high computational cost of full-model optimization problems, some methods leverage a gradation of model fidelities along a time horizon \cite{khazoom2024humanoid, li2024cafempccascadedfidelitymodelpredictive}. Other methods rely on offline trajectory optimization to generate desirable behaviors, and then track these behaviors online \cite{westervelt_hybrid_2003}. For underactuated systems, the online tracking problem can be non-trivial, often requiring additional feedback mechanisms to stabilize the underactuated states such as regulators \cite{reher2021dynamic}. 
%

Reinforcement learning (RL) \cite{Schulmanetal_ICLR2016} takes the concept of offline computation even further, using concepts from stochastic optimal control and parallelized simulation environments to synthesize feedback controllers. RL methods have shown robust performance \cite{miki2022learning, li2024reinforcementlearningversatiledynamic} when the policy is trained in sufficiently randomized domains. 
Current methods in RL improve policies through simulator rollouts \cite{suh2022differentiable}, typically at the expense of high data complexity. 
%
%
Although these can work well, they exhibit extreme sensitivity to cost function parameters and ignore the underlying system structure.

Heuristics, on the other hand, are able to leverage intuition about system structure, and can achieve stabilization with minimal online or offline computational overhead. In the context of legged locomotion, the Raibert Heuristic for hopping \cite{raibert_experiments_1984}, inverted pendulum models for walking \cite{kajita20013d}, and spring-loaded pendulums for running \cite{han20223d} all reason about where a legged robot's feet should be placed in order to stabilize the center of mass. While these methods may be less formal than the methods above and require significant domain expertise to implement, they tend to reason (perhaps implicitly) about the fundamental control structure needed to address the underactuation. 

The above methods generally intersect in two places: first, an application of feedback to the actuated states based on the position of underactuated states (either explicitly or through replanning), and second, a dependence on optimality to generate stable, desirable behaviors. We propose a method which combines these two ideas, using optimality to ensure stability while reasoning explicitly about the structure of underactuation. Specifically, we leverage the notion of \textit{zero dynamics} to explicitly decompose the system into actuated and unactuated coordinates \cite{isidori_elementary_1995, westervelt2003hybrid, reher2021control, da2017combiningtrajectoryoptimizationsupervised}. We pair this paradigm with optimal control to learn a mapping from the unactuated state to a desired actuated state, termed a Zero Dynamics Policy (ZDP), which is then stabilized using a tracking controller. 
This perspective aligns with prior work on Hybrid Zero Dynamics (HZD) \cite{westervelt2003hybrid}; however, rather than assuming stability of the zero dynamics manifold or relying on phasing variables and periodicity, we use optimal control to provably and constructively synthesize stable output-zeroing manifolds.


We propose a general framework for the control of hybrid underactuated systems and apply it to hopping, which exemplifies the challenges of such systems due to the large number of passive degrees of freedom, tight input constraints, and short ground phases.
Our empirical validation of ZDPs on the ARCHER 3D hopping robot showcases an agile and stable controller as seen in Figure~\ref{fig:experiments} and the supplemental video \cite{noauthor_supplemental_video_2022}. Over the course of more than 3000 hops, our method achieves state of the art disturbance rejection, hops over long distances on a treadmill, navigates an obstacle course and rough terrain without vision, and is precise enough to reliably hop across narrow bridges.

\section{Preliminaries}
\subsection{Hybrid Dynamics and Lyapunov Stability}
Consider an $n-$degree of freedom robotic system with coordinates $\b q\in\mathcal{Q}$ and state $\b x = (\b q, \dot{\b q})\in \mathcal{X} \triangleq \mathsf{T}\mathcal{Q}.$ Using the Euler Lagrange equations, we write the continuous-time dynamics in control-affine form as:
\begin{align}
    \dot{\b x} = \underbrace{\begin{bmatrix} \dot {\b q} \\ -\b D(\b q)^{-1} \b H(\b q, \dot{\b q}) \end{bmatrix}}_{\b f(\b x)} + \underbrace{\begin{bmatrix} \b 0 \\ \b D(\b q)^{-1} \b B\end{bmatrix}}_{\b g(\b x)}\b u \label{eq:robot_dynamics}
\end{align}
where $\b D : \mathcal{Q} \to \R^{n\times n}$ is the positive-definite mass-inertia matrix, $\b H:\mathcal{X} \to\R^n$ contains the Coriolis and gravity terms, $\b B\in\R^{n\times m}$ is the selection matrix, and $\b u\in\R^m$ is the control input. For the following discussion we assume that $\b B$ has (column) rank $m < n$, i.e. \eqref{eq:robot_dynamics} is underactuated.

As the robot experiences impulsive effects, it is subject to the instantaneous momentum transfer equation:
\begin{align}
    \b {x^+} = \b \Delta(\b x^-), \label{eqn:disc}
\end{align}
with $\b \Delta:\mathcal{X}\to\mathcal{X}$ representing the impact map. Combining \eqref{eq:robot_dynamics} and \eqref{eqn:disc}, the complete hybrid dynamics can be written as:
\begin{align*}
    \mathscr{H} = \begin{cases}
        \dot{\b x} = \b f(\b x) + \b g(\b x) \b u &~~~ \b x \notin \mathcal{S}\\
        \b x^+ = \b\Delta(\b x^-) &~~~ \b x^- \in \mathcal{S}
    \end{cases}
\end{align*}
where $\mathcal{S}\subset\mathcal{X}$ is an appropriately defined switching surface, for example the foot making or breaking contact with the ground \cite{westervelt_hybrid_2003}.



Towards developing a stabilizing feedback controller for \eqref{eq:robot_dynamics}, define a collection of continuous time outputs ${\b y:\mathcal{X} \to \R^m}$ that we would like to drive to zero. For outputs of relative degree two \cite{sastry_linearization_1999}, consider the error coordinates $\b e = (\b y, \dot{\b y})\in \mathcal{E}\subseteq\R^{2m}$. These errors can be constructively stabilized via a RES-CLF, defined as:
\begin{definition} \cite{ames2017hybrid}
    For the system \eqref{eq:robot_dynamics}, $V_\varepsilon: \mathcal{E} \rightarrow \R$ is said to be a \textit{rapidly exponentially stabilizing control Lyapunov function (RES-CLF)} if there exists a $\lambda,k_1,k_2 > 0$, such that for all $\varepsilon \in (0, 1)$:
    \begin{align}
        k_1 \|\b e\|^2 \leq V_\varepsilon(\b e) &\leq k_2 \|\b e\|^2 \nonumber \\ 
        \inf_{\b u} \dot V_\varepsilon(\b x, \b u)&\leq -\frac{\lambda}{\varepsilon} V_\varepsilon(\b e).\label{eq:Lyap_stability_cont}
    \end{align}
\end{definition}

Valid relative degree ensures the existence of a nonempty set $\mathcal{K}$, defined to be the set of all controllers satisfying the inequality \eqref{eq:Lyap_stability_cont}. 
Any controller $\b k \in \mathcal{K}$ renders the continuous time output exponentially stable, i.e. there exists $M, \Tilde{\lambda} > 0$ such that:
\begin{align*}
    \|\b e(t)\| \leq Me^{-\frac{\Tilde{\lambda}}{\varepsilon}t}\|\b e(0)\|,
\end{align*}
whereby tuning $\varepsilon$ down enables arbitrarily fast convergence. 

\subsection{From Hybrid Dynamics to Discrete-Time Dynamics}

We will be interested in modeling $\mathscr{H}$ as a discrete-time dynamical system via its impact-to-impact dynamics. To this end, let $\b x_k\in\mathcal{X}$ denote the robot state just before impact, $\paramspace$ denote an admissible parameter set for $\inputk\in\paramspace$, a discrete parameterization of the control input over a single continuous phase, and $t_k\in\R_{\ge 0}$ be the duration of the continuous phase. We reformulate our hybrid control system into discrete dynamics via:
\vspace{-2mm}
\begin{align}
    \b x_{k+1} = \b F(\b x_k, \inputk), \label{eq:disc_dyn}
\end{align}
where $\b F:\mathcal{X}\times \paramspace \to \mathcal{X}$ composes the the impact map \eqref{eqn:disc} with the flow of \eqref{eq:robot_dynamics} under a parameterized feedback controller ${\b u = \b k(\b x(t), \inputk) \in \mathcal{K}}$.
In the context of hopping, we take $\inputk$ to be the desired impact angle. This parameterization of control input allows us to reason about the effect of impact conditions on the resulting system dynamics, which are the primary means of stabilizing legged systems.
%
%
 %
 Note that here we assume the existence of a lower bound between impact times so that $\b F$ is well defined.
 For a complete discussion of how to achieve this representation from the underlying hybrid dynamics, see  \cite{da2017combiningtrajectoryoptimizationsupervised}. Similar to the continuous-time case, the stability of the discrete time error dynamics can be reasoned about via Lyapunov theory:

 \begin{definition}\label{def:disc_lyap}
    For the system $\b e_{k+1} = \b F(\b e_k)$, $V:\mathcal{E}\to \R$ is a \textit{discrete exponential Lyapunov function} if it is positive definite and there exists an $\alpha \in (0, 1], \ k_1, k_2 > 0$ such that:
    \begin{align*}
        k_1 \|\b e_k\|^2 \leq V(\b e_k) &\leq k_2 \|\b e_k\|^2 \nonumber \\
        \Delta V(\b e) = V(\b e_{k+1}) - V(\b e_k) &\le -\alpha V(\b e_k).
    \end{align*}
\end{definition}

The existence of such a Lyapunov function is necessary and sufficient for exponential stability of a system, i.e. the existence of $M>0, \ \beta \in [0,1)$ such that:
\begin{align*}
    \|\b e_{k} \| \leq M \beta^k \| \b e_0\|.
\end{align*}


\subsection{Discrete-Time Optimal Control}
We leverage optimal control to synthesize inputs $\inputk$ which stabilize the discrete time system in \eqref{eqn:disc} while satisfying input constraints. To this end, consider the following infinite-time optimal control problem:
\begin{align}
    V(\b x_0) \triangleq \min_{\substack{\b x_k, \inputk}}\quad & \sum_{k=0}^\infty c(\b x_k, \inputk) \label{eqn:ftocp}\\
\textrm{s.t.} \quad & {\b x}_{k+1} = {\b F}(\b x_k, \inputk) \notag\\
    \quad & \b h(\b x_k, \inputk)\le 0\notag
\end{align}
where $V: \mathcal{X} \rightarrow \R$ is termed the value function, $c:\mathcal{X} \times \paramspace \to \R$ is a positive-definite cost function and $\b h:\mathcal{X}\times\paramspace\to\R^p$ contains any state-input constraints. With this, we can define the state-action value function ${Q:\mathcal{X} \times \paramspace \to \R}$ as:
\begin{align*}
    Q(\b x_k, \inputk) = c(\b x_k, \inputk) + V(\b x_{k+1}),
\end{align*}
which defines the optimal control input at any state $\b x_k$ through following optimization program:
\begin{align}
    \inputk^*(\b x_k) = \text{arg}\min_{\inputk}~& Q(\b x_k, \inputk) \label{eqn:q_func}\\ 
    \text{s.t.} ~~& \b h(\inputk, \b x_k) \le \b 0 \notag
\end{align}
We rely on iteratively solving convex approximations of this nonconvex problem via iLQR. In Section~\ref{sec:method} we show that tracking the output of optimal controllers in continuous time results in exponential stability of the discrete time dynamics.


\subsection{Outputs and Zero Dynamics}
Understanding the structure of underactuation provides key insight into constructing stabilizing controllers for these systems.
To analyze the states that actuation directly impacts, consider the following coordinate change:
\begin{align}
    \b \eta = \b\Phi_{\b\eta}(\b x) \triangleq \begin{bmatrix}\b B^\top \b q\\ \b B^\top \dot{\b  q}\end{bmatrix},~~~\b z = \b\Phi_\b z(\b x) \triangleq\begin{bmatrix} \b N \b q \\ \b N \b D(\b q) \dot{\b q}\end{bmatrix} \label{eqn:outputs}
\end{align}
for $\b\eta \in \outputSpace\subset\mathcal{X}$ and $\b z \in \zeroSpace\subset\mathcal{X}$, where $\b N\in\R^{(n-m)\times n}$ is chosen to be a basis for the left nullspace of $\b B$. It is easily verified that the coordinate change ${\b \Phi(\b x) \triangleq (\b \Phi_{\b \eta}(\b x), \b \Phi_{\b z}(\b x))}$ is a diffeomorphism between $\mathcal{X}$ and $\outputSpace\times \zeroSpace$; therefore, $\b\Phi^{-1}$ exists and any conclusions of stability of $(\b\eta, \b z)$ are directly transferable back to $\b x$.
In these coordinates, the hybrid dynamics are given by:
\begin{align*}
    &\dot{\b \eta}= \hat{\b f}(\b \eta, \b z) +\hat{\b g}(\b \eta, \b z) \b u, \ \ \ \ 
    \dot{\b z} = \b \omega(\b \eta, \b z), \ \ \  \b\Phi^{-1}(\b \eta, \b z)\notin \mathcal{S} \\
    &\b\eta^+ = \b\Delta_{\b\eta}(\b\eta^-, \b z^-), \ \ \
    \b z^+ = \b\Delta_{\b z}(\b\eta^-, \b z^-), \ \ \b\Phi^{-1}(\b \eta, \b z)\in \mathcal{S}
\end{align*}
termed the \textit{actuated} dynamics and the \textit{unactuated} dynamics, respectively.
Note that these coordinates were exactly chosen such that $\hat{\b g}(\b \eta, \b z)$ is full rank and $\frac{d \b z}{d \b x} \b g(\b x) \equiv \b 0$; as such, this mapping decomposes the state space into coordinates which can directly be controlled, and those which cannot.


Assuming the continuous time input does not effect the impact map or impact time\footnote{This assumption is needed so that $\b \Omega$ is not a function of $\b v_k$ and is well justified on ARCHER as impact angle weakly effects impact time.}, applying $\b \Phi$ to the discrete dynamics \eqref{eq:disc_dyn} results in:
\begin{align}
    \b \eta_{k+1} = \hat{\b F}(\b \eta_k, \b z_k, \inputk),\ \ \
    \b z_{k+1} = \b \Omega(\b \eta_k, \b z_k). \label{eqn:disc_normal_form}
\end{align}

Now, consider a mapping $\b \psit : \zeroSpace\to \outputSpace$ and associated discrete-time error $\b e_k = \b \eta_k - \b \psit(\b z_k)$. The goal will be to design $\b\psit$ such that driving $\b e_k$ to zero results in stability of the overall system. This choice of error parameterization is inspired by other successful results in robotics; the Raibert Heuristic \cite{raibert_experiments_1984}, reduced order models \cite{han20223d}, and regulators for HZD gaits \cite{reher2021control} all reason about where to place a robot's feet (the actuated state) as a function of their center of mass state (the underactuated state). 
We aim to generalize these methods and reason explicitly about constructive methods to generate provably stable behaviors. 
The construction of the mapping $\b \psit$ induces an associated manifold $\mathcal{M}_{\b\psi}\subset\mathcal{X}$ via:
\begin{align}
    \mathcal{M}_{\b\psi} \triangleq \{(\b \eta_k, \b z_k)~|~ \b \eta_k = \b \psit(\b z_k)\}. \label{eqn:zeroing_manifold}
\end{align}
We will be interested in enforcing conditions such that $\mathcal{M}_{\b\psi}$ is controlled invariant, defined as:
\begin{definition} \label{def:inv_manifold}
    The manifold $\mathcal{M}_{\b\psi}$ is \textit{controlled invariant} if for all $(\b\eta_k, \b z_k)\in \mathcal{M}_{\b\psi}$ there exists a $\inputk \in \paramspace$ such that the next state remains on the manifold, i.e.:
    \vspace{-1mm}
    \begin{align*}
        \Big(\b F(\b\eta_k, \b z_k, \inputk),~ \b \Omega(\b \eta_k, \b z_k)\Big) \in\mathcal{M}_{\b\psi}.
    \end{align*}
    \vspace{-3mm}
\end{definition}
\vspace{-3mm}
Assuming a controlled invariant manifold $\mathcal{M}_{\b\psi}$, we now have the notion of discrete-time zero dynamics:
\begin{definition}
    The \textit{discrete-time zero dynamics} associated with a controlled invariant manifold $\mathcal{M}_{\b \psi}$ are given by:
    \begin{align*}
        \b z_{k+1} = \b \Omega(\b \psit(\b z_k), \b z_k).
    \end{align*}
\end{definition}

These dynamics are autonomous but determined by choice of $\b\psit$; therefore, the goal of this work will be to design $\b\psit$ such that the zero dynamics are stable. We show that stability on $\mathcal{M}_{\b\psi}$ paired with a suitably defined output controller results in stability of the overall system.

\section{Discrete-Time Zero Dynamics Policies}
\label{sec:method}

We propose a discrete-time mapping from the underactuated state, $\b z_k$, to a desired actuated state, $\b \eta_k$. This mapping, $\b \psit: \zeroSpace \rightarrow \outputSpace$, will encode the desired position of the actuated coordinates given the location of the unactuated coordinates at impact.
The job of the continuous time controller is to drive $\b \eta(t)$ to the desired preimpact location, $\b \psit(\b z_{k+1})$.

In this section, we will first reason about the ability of continuous time controllers to render $\mathcal{M}_{\b\psi}$ attractive and invariant by driving the error $\b e$ to zero. Second, we demonstrate that if the manifold has stable zero dynamics (trajectories on the manifold converge to the origin), then stabilizing the manifold stabilizes the entire system. Finally, we propose a learning pipeline which leverages optimal control to find a manifold with the desired properties. 

\subsection{Constructive Stabilization of the Zeroing Manifold}
We show that the structure of the proposed manifold allows constructive stabilization techniques:
\begin{lemma}\label{thm:stabilize_manifold}
\textit{    Consider a controlled invariant manifold $\mathcal{M}_{\b\psi}$. There exists a continuous-time control law $\b k\in\mathcal{K}$ which results in exponential stabilization of $\|\b \eta_k -\b \psit(\b z_k)\|$.}
\end{lemma}
\begin{proof}
    Consider a point $(\b \eta_k, \b z_k)$ and the evaluation of the current and next states on the manifold: $\b \psit(\b z_k)$ and $\b \psit(\b z_{k+1})$, respectively.
    As the $\b \eta(t)$ dynamics are feedback linearizable, there exists a dynamically feasible trajectory $\b \eta_d(t)$ such that $\b \eta_d(0) =  (\b \psit(\b z_k))^+$, and $\b \eta_d(t_k) = \b \psit(\b z_{k+1})$, where $t_k$ is the impact time and $(\cdot)^+$ denotes a postimpact state. For example, $\b \eta_d(t)$ can be constructed using Bezier polynomials \cite{csomay2022multi}. Using a controller $\b k\in\mathcal{K}$, i.e. satisfying the RES-CLF condition \eqref{eq:Lyap_stability_cont}, we can obtain exponential convergence to this trajectory in continuous time:
    \begin{align*}
        \|\b \eta(t) - \b \eta_d(t)\| \leq Me^{-\frac{\lambda}{\varepsilon}t}\|\b \eta_k^+ - (\b \psit(\b z_{k}))^+\|,
    \end{align*}
    for $M, \lambda > 0$. Taking $T_*>0$ to be the lower bound between impact times, the impact states are uniformly bounded by:
    \begin{align*}
        \|\b \eta_{k+1} - \b \psit(\b z_{k+1})\| \leq Me^{-\frac{\lambda}{\varepsilon}T_*}\|\b \eta_k^+ - (\b \psit(\b z_{k}))^+\|.
    \end{align*}
    Then, using the properties of the impact map we have:
    \begin{align*}
        \|\b \eta_k^+ - (\b \psit(\b z_{k}))^+\| &= \| \b \Delta_{\b \eta}(\b \eta_k, \b z_k) - \b \Delta_{\b \eta}(\b \psit(\b z_k), \b z_k) \| \\
        &\leq L_{\Delta} \|\b \eta_k - \b \psit(\b z_k)\|,
    \end{align*}
    substituting into the bound above, and choosing $\varepsilon > 0$ sufficiently small that $\alpha = ML_{\Delta}e^{-\frac{\lambda}{\varepsilon}T_*} \in (0, 1]$, we have:
    \begin{align*}
         \|\b \eta_{k+1} - \b \psit(\b z_{k+1})\| \leq \alpha \|\b \eta_k - \b \psit(\b z_k) \|,
    \end{align*}
    proving exponential stability to the manifold, as desired.     
    \qed
\end{proof}
\begin{remark}
    The desired trajectory $\b \eta_d(t)$ is being implicitly replanned at impact via $\b \psit$ as a function of the underactuated state $\b z_{k}$. Additionally, the manifold $\mathcal{M}_{\b\psi}$ is invariant under the discrete dynamics $\b F$, but is notably not hybrid invariant.
\end{remark}
\subsection{Composite Stability}
The previous section demonstrated a method for constructing a controller to exponentially stabilize the system to a controlled invariant manifold $\mathcal{M}_{\b\psi}$. We now show that exponentially stabilizing the system to a manifold with stable zero dynamics results in composite exponential stability of the entire system:
%
%
\begin{theorem} \label{thm:composite_stability}
    \textit{Consider a controlled invariant manifold $\mathcal{M}_{\b\psi}$ whose zero dynamics are exponentially stable. Any control law exponentially stabilizing $\b \|\b \eta_k - \b \psit(\b z_k)\|$ stabilizes the discrete-time composite system $(\b \eta_k, \b z_k)$ to the origin. }
\end{theorem}
\begin{proof}
    Define $\b e_k = \b \eta_k - \b \psit(\b z_k)$. By Lemma 1, there exists a continuous-time controller $\b k \in \mathcal{K}$ rendering the discrete error dynamics exponentially stable. As such, converse Lyapunov theory guarantees the existence of a Lyapunov function $V_{\b e}:\mathcal{E}\to\R$ satisfying:
    \begin{align*}
        k_1 \|\b e_k\|^2 \leq V_{\b e}(\b e_k) &\leq k_2 \|\b e_k \|^2 \\
        \Delta V_{\b e}(\b e_k) &\leq -k_3 \| \b e_k\|^2
    \end{align*}
    Similarly, the stability of $\mathcal{M}_{\b\psi}$ implies the existence of a Lyapunov function $V_{\b z}:\zeroSpace\to \R$ satisfying:
    \begin{align*}
        k_4 \|\b z_k\|^2 \leq V_{\b z}(\b z_k) &\leq k_5 \|\b z_k\|^2 \\
        \Delta V_{\b z}(\b z_k) = V_{\b z}(\b\Omega(\b\psit(\b z_k), \b z_k)) - V_{\b z}(\b z_k) &\leq -k_6 \| \b z_k\|^2
    \end{align*}
    The Lyapunov function $V_{\b z}$ will additionally satisfy \cite{ames2017hybrid}:
    \begin{align*}
        |V_{\b z}(\b z) - V_{\b z}(\b z')| & \leq k_7\|\b z - \b z'\|\left(\|\b z\| - \|\b z'\|\right) \triangleq \Gamma(\b z, \b z').
    \end{align*}
    \begin{figure}
    \vspace{2.5 mm}
    \centering
    \includegraphics[width=\columnwidth]{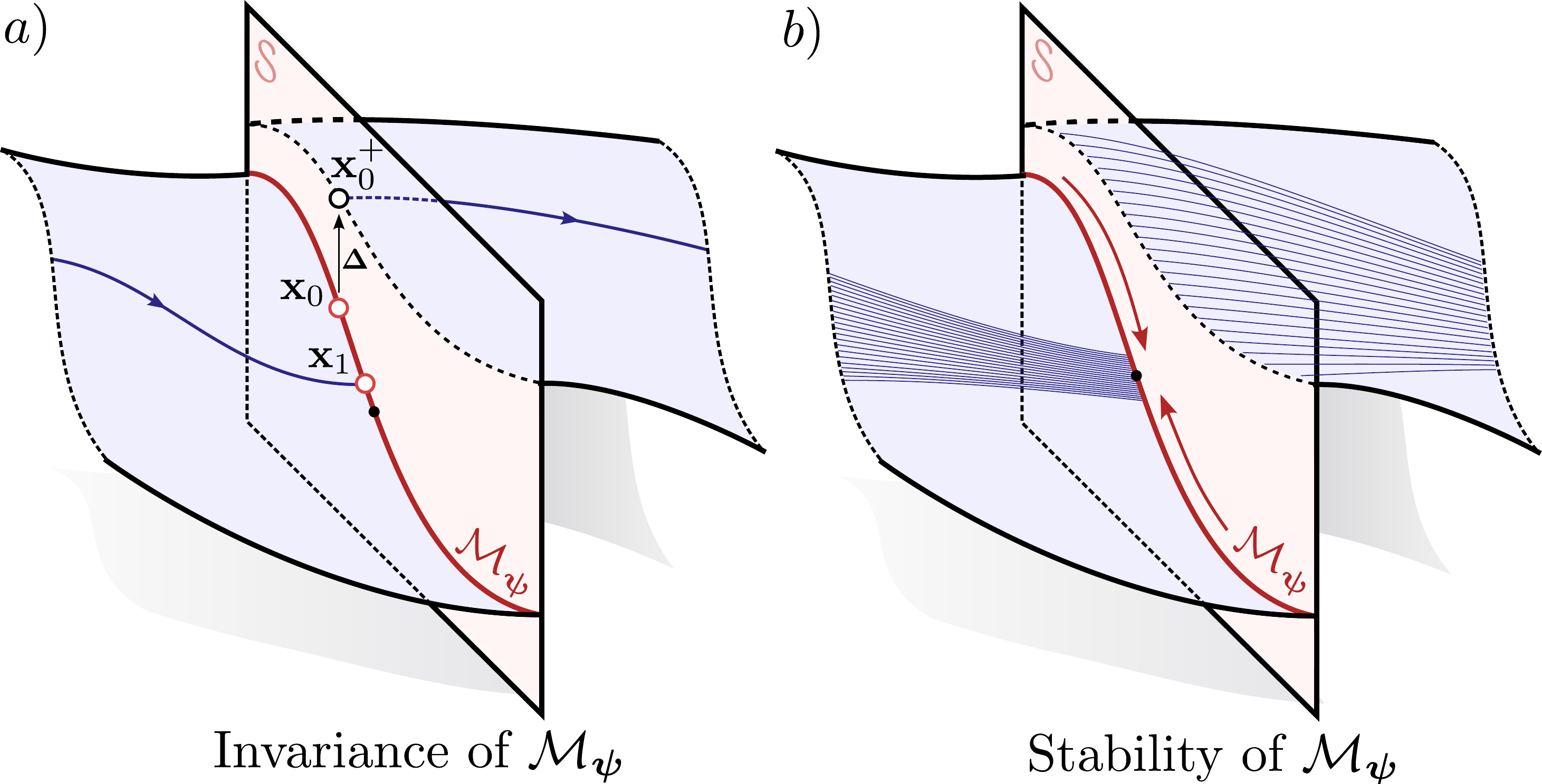}
    \caption{A depiction of the two necessary properties of $\mathcal{M}_{\b \psi}$: a) invariance under the discrete map $\b F$, and b) stability.}
    \label{fig:theory}
    \vspace{-4mm}
\end{figure}
    \noindent Consider the composite Lyapunov function candidate $V(\b e_k, \b z_k) \triangleq \sigma V_{\b e}(\b e_k) + V_{\b z}(\b z_k)$ with $\sigma>0$, whereby:
    \begin{align*}
        \min\{\sigma k_1, k_4\}\|\b e, \b z\|^2 \leq V(\b e, \b z) \leq \max\{\sigma k_2, k_5\} \|\b e, \b z\|^2.
    \end{align*}
    Furthermore, since $\b z_k$ is exponentially stable on $\mathcal{M}_{\b \psi}$, discrete sequences on $\mathcal{M}_{\b \psi}$ will be exponentially decreasing:
    \begin{align*}
        \|\b z_{k+1}\| = \|\b\Omega(\b \psit(\b z_k), \b z_k) \| \leq M\lambda \|\b z_k\|,
    \end{align*}
    for $\lambda \in [0,1)$ and $M > 0$. Compute the difference of $\Delta V$:
    \begin{align*}
        \Delta V &= \sigma \Delta V_{\b e}(\b e_k) + V_{\b z}(\b \Omega(\b \eta, \b z_k)) - V_{\b z}(\b z_k) \\
        &= \sigma \Delta V_{\b e}(\b e_k) + \Delta V_{\b z}(\b z_k)\\&\quad\quad+ V_{\b z}(\b \Omega(\b \eta_k, \b z_k)) - V_{\b z}(\b \Omega(\b \psit(\b z_k), \b z_k)) \\
        &\leq -\sigma k_1 \|\b e_k\|^2 - k_6 \|\b z_k\|^2 \\&\quad\quad+ \Gamma(\b \Omega(\b \eta_k, \b z_k), \b \Omega(\b \psit(\b z_k), \b z_k)) \\
        &= -\sigma k_1 \|\b e_k\|^2 - k_6 \|\b z_k\|^2 \\&\quad\quad +k_7L_{\b \Omega}^2 \|\b e_k \|^2 + 2M\lambda k_7L_{\b \Omega} \|\b e_k\|\|\b z_k\|\\
        &= -\begin{bmatrix}
            \|\b e_k\| \\ \|\b z_k\|
        \end{bmatrix}^\top \begin{bmatrix}
            \frac{\sigma k_1}{2} - c(\sigma) & -M\lambda k_7 L_{\b \Omega} \\ -M\lambda k_7 L_{\b \Omega} & k_6
        \end{bmatrix}\begin{bmatrix}
            \|\b e_k\| \\ \|\b z_k\|
        \end{bmatrix} 
    \end{align*}
    where $c(\sigma) =  k_7L_{\b \Omega}^2 - \frac{\sigma}{2}k_1$, and $\Gamma(\b \Omega(\b \eta, \b z), \b \Omega(\b \psit(\b z), \b z))$ is bounded using Lipschitz properties of the dynamics. Choosing $\sigma > \max\left\{\frac{2M^2\lambda^2k_7^2 L_{\b \Omega}^2}{k_1k_6}, \frac{2k_7 L_{\b \Omega}^2}{k_1}\right\}$ ensures the matrix is positive definite; therefore, $V$ is a Lyapunov function certifying composite stability.
    \qed
\end{proof}

\begin{remark}
    Figure \ref{fig:theory} depicts each of the assumptions used to prove stability in \cref{thm:composite_stability}, namely discrete invariance and exponential stability of $\mathcal{M}_{\b \psi}$. Subsequent sections will develop constructive techniques leveraging optimal control and learning for finding such manifolds. 
\end{remark}

\subsection{Stability via Optimal Control}

We will leverage optimality to enforce the stability on $\mathcal{M}_{\b\psi}$. This choice is motivated by the fact that asymptotic stability is a necessary condition for an optimal controller to be well defined \cite{liberzon_calculus_2012}. As \cref{thm:composite_stability} rests on assumptions of exponential stability, we define conditions under which optimality implies exponential stability:

\begin{theorem}
    \label{thm:exp_stab_value}
   \textit{ Let $V(\b x_k)$ be the value function for the optimal control problem defined in \cref{eqn:ftocp}, where the cost function is quadratic, $c(\b x_k, \b v_k) = \b x_k^\top \b Q \b x_k + \b v_k^\top \b R \b v_k$, and the domain $\mathcal{X}$ is compact. If there exists an $\varepsilon > 0$ such that the LQR approximation of \cref{eqn:ftocp} taken by linearizing the dynamics around the equilibrium point satisfies: }
    \begin{align}
        \b v_{LQR}(\b x_k) = -\b K \b x_k \in \inputSetk \quad \forall \b x_k \in B_\varepsilon(\b 0), \label{eqn:lqr}
    \end{align} 
    \textit{with $\mathcal{H}(\b x_k) \triangleq \{\b v_k \in P~|~\b h(\b x_k, \b v_k)\le 0\}$, then the nonlinear system is exponentially stable under the optimal controller.}
\end{theorem}
\begin{proof}
    We begin by showing the optimal controller \cref{eqn:ftocp} is exponentially stabilizing in a neighborhood of the origin. Then, we extend this claim to the entire state space.
    In a sufficiently small ball around the origin, LQR \eqref{eqn:lqr} will be exponentially stabilizing for the nonlinear system \cite{sastry_linearization_1999}, as it locally satisfies input bounds. This implies constants $M_{\textrm{LQR}}, \delta>0$ and $\lambda_{\textrm{LQR}} \in [0,1)$ such that:
    \begin{align*}
        \|\b x_k\| \leq M_{\textrm{LQR}}\lambda_{\textrm{LQR}}^k \|\b x_0\| \quad \forall \b x_0 \in B_\delta(\b 0), \ \ \forall k \in \mathbb{Z}_+.
    \end{align*}
    We first show that the optimal trajectory emanating from an initial condition $\b x_0 \in B_\delta(\b 0)$ is similarly exponentially stable. For any $M > 0, \ \lambda \in (0, 1)$, consider two cases:
    
    \proofsubstep{Case 1:} There exists a finite index set $\{k_i\}_{i=0}^N$ satisfying:
    \begin{align*}
        \|\b x_{k_i}\| \geq M\lambda^{k_i} \|\b x_{0}\|.
    \end{align*}
    Compute the maximum violation ratio $R\geq1$ given by:
    \begin{align*}
        R \triangleq \max_{i \in \{0, \hdots, N\}} \frac{\|\b x_{k_i}\|}{M\lambda^{k_i} \|\b x_0\|}.
    \end{align*}
    If the index set is empty, take $R = 1$. Then
    \begin{align*}
        \|\b x_k\| \leq R M\lambda^{k} \|\b x_0\| \quad \forall k \in \mathbb{Z}_+
    \end{align*}
    And the trajectory is exponentially stable. 

    \proofsubstep{Case 2:} There exists a infinite index set $\{k_j\}_{j=0}^\infty$ satisfying:
    \begin{align} \label{eqn:inf_violation}
        \|\b x_{k_j}\| \geq M\lambda^{k_j} \|\b x_0\|.
    \end{align}
   We will establish that $V(\b x_k)$ is an exponential Lyapunov function (\cref{def:disc_lyap}) along the trajectory, and thus the trajectory is exponentially stable. First, we bound the value function difference:
    \begin{align}
        \Delta V(\b x_k) &= V(\b x_{k}) - V(\b x_{k-1})\nonumber = -\b x_k^\top \b Q \b x_k - \b v_k^\top \b R \b v_k \nonumber\\
        &\leq -\lambda_{\min}(\b Q) \|\b x_k\|^2 \label{eqn:deltaV_bound}
    \end{align}
    Next, we need to show that $V(\b x_k)$ is bounded by quadratics. 
    Because the LQR controller is suboptimal for the nonlinear system, applying it increases the cost relative to $V(\b x_k)$:
    \begin{align*}
        V(\b x_0) &\leq \sum_{k=0}^\infty \b x_k^\top \b Q \b x_k + (\b K\b x_k)^\top \b R (\b K\b x_k) \\
        &\leq \sum_{k=0}^\infty \left(\bar{\lambda}(\b Q) + \bar{\lambda}(\b K^\top \b R \b K)\right) \|\b x_k\|^2 \\
        &\leq \sum_{k=0}^\infty \left(\bar{\lambda}(\b Q) + \bar{\lambda}(\b K^\top \b R \b K)\right) M_{\textrm{LQR}}^2 \lambda_{\textrm{LQR}}^{2k}\|\b x_0\|^2 \\
        &=  \frac{M_{\textrm{LQR}}^2}{1-\lambda_{\textrm{LQR}}^2}\left(\bar{\lambda}(\b Q) + \bar{\lambda}(\b K^\top \b R \b K)\right) \|\b x_0\|^2
    \end{align*}
    where $\underline\lambda$ and $\overline\lambda$ are the minimum and maximum eigenvalue oeprators, respectively.
    
    Finally, using \eqref{eqn:inf_violation}, we can lower bound $V(\b x_k)$ by:
    \begin{align*}
        V(\b x_0) &= \sum_{j=0}^\infty \b x_{k_j}^\top \b Q \b x_{k_j} + \b v_{k_j}^\top \b R \b v_{k_j} \\
        &\geq \sum_{j=0}^\infty \underline{\lambda}(\b Q) \|\b x_{k_j}\|^2 \\
        &\geq \sum_{j=0}^\infty \underline{\lambda}(\b Q) M^2\lambda^{2k_j} \|\b x_0\|^2 \\
        &= \left[\frac{M^2}{1-\lambda^2}\left(\bar{\lambda}(\b Q) + \bar{\lambda}(\b K^\top \b R \b K)\right) - c\right] \|\b x_k\|^2
    \end{align*}
    Where $c$ is the sum of the terms removed from the geometric series. Lastly, 
    The above bounds hold for each point on the trajectory; therefore, $V$ is a Lyapunov function certifying exponential stability of the trajectory. 

    Finally, we extend the claim outside of the ball around the origin. As $V \succ 0$ and $\Delta{V} \prec 0$, the optimal controller is asymptotically stable \cite{liberzon_calculus_2012}. By compactness of $\mathcal{X}$ and \eqref{eqn:deltaV_bound}, the time to enter $B_\delta(\b 0)$ is bounded by:
    \begin{align*}
        {K} \triangleq \frac{\sup_{\b x_0 \in \mathcal{X}} V(\b x_0)}{\inf_{\b x_0 \in \mathcal{X}\backslash B_\delta(\b 0)} \Delta V(\b x_0)} \leq \frac{\sup_{\b x_0 \in \mathcal{X}} V(\b x_0)}{\underline\lambda(\b Q) \delta^2}.
    \end{align*}
    Because trajectories converge exponentially in $B_\delta(\b 0)$,
    \begin{align*}
         \|\b x_k\| \leq M\lambda^{k-{K}} \|\b x_{{K}}\| \quad \forall \b x_0 \in B_\delta(\b 0), \ \ \forall k \geq {K}
    \end{align*}
    for $M>0, \ \lambda \in [0, 1)$. By compactness of $\mathcal{X}$, trajectories are uniformly bounded $\|\b x_k\| \leq B$; therefore:
    \begin{align*}
        \|\b x_k\| \leq \frac{\max\{B, M\}\lambda^{-K}}{\min\{1, \delta\}}\lambda^k \|\b x_0\|\quad \forall k \in \mathbb{Z}_+
    \end{align*}
    is an exponential upper bound for the entire trajectory.
    \qed
\end{proof}

\subsection{Constructing the Zeroing Manifold via Learning}
By \Cref{thm:exp_stab_value}, a manifold which is invariant under the optimal controller will be exponentially stable. Such a manifold then satisfies the assumptions of \cref{thm:composite_stability} and can be constructively stabilized resulting in composite stability of the entire system.

We will now present a learning method which leverages optimal control to ensure the assumptions of controlled invariance and stability of $\mathcal{M}_{\b\psi}$ as depicted in Figure~\ref{fig:theory} are met.
Specifically, we will search for a manifold that is invariant under the optimal action, i.e. the controller that keeps sequences of states in the manifold coincides with the optimal controller for \cref{eqn:ftocp}.

To concisely define the loss function consider the variable 
\begin{align}
    \b \zeta_{\b \theta}(\b z) \triangleq \begin{bmatrix} \b \psit(\b z) \\ \b z\end{bmatrix}
\end{align}
which encodes a point on the manifold. The loss function is:
\begin{align} \label{eqn:loss}
    \mathcal{L}(\b\theta) = \mathop{\mathbb{E}}_{\b z \sim \text{UNIFORM}} \left\lVert \b\eta_{1}^*\left(\b \zeta_{\b \theta}(\b z)\right) - \b\psit\left(\b z^*_{1}\left(\b \zeta_{\b \theta}(\b z)\right)\right)\right\rVert_2^2,
\end{align}
where $\b z_{1}^* = \b \Omega(\b \psi(\b z), \b z)$ and $\b \eta_{1}^* = \hat{\b F}(\b \psi(\b z), \b z, \b v^*)$, with $\b v^*$ the optimal control input. The expectation is taken over a uniform distribution over $\zeroSpace$. The loss function directly measures how far an initial condition on the manifold deviates from the manifold under one discrete step of the optimal controller as depicted in \cref{fig:learning}. 

The learning pipeline outlined in Algorithm \ref{alg:train} starts an epoch by sampling a batch of points from $\zeroSpace$, therefore enabling a dimension reduction as compared to the complete state space. The network is then evaluated to produce a set of points on the current manifold, $\{\b\zeta_{\b\theta}(\b z_i)\}_{i=1}^{N}$. We then approximately solve the optimal control problem \cref{eqn:ftocp}. Finally, we simulate the system forwards one step to obtain $\left(\b \eta^*_{1}, \b z^*_{1}\right)$ which the loss computation in \cref{eqn:loss} requires. If $\b \psit$ attains zero loss, because of continuity of the network and the loss function we can conclude that the resulting manifold $\mathcal{M}_{\b \psi}$ is invariant under the optimal control and can render the full order system stable by satisfaction of the preconditions for \cref{thm:composite_stability}.

\begin{figure}[t!]
    \centering
    \vspace{5pt}
    \includegraphics[width=\columnwidth]{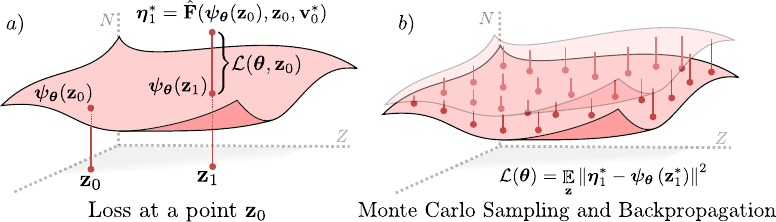}
    \caption{a) The loss function exactly measures the extent to which the manifold is not invariant under optimal action b) a Monte Carlo approximation of the spatial loss is used, wherein the optimal policy is backpropogated through to update the surface.
    }
    \label{fig:learning}
    \vspace{-4mm}
\end{figure}


\section{Application of ZDP to ARCHER}
We deployed the ZDP method on the 3D hopping robot ARCHER. To discuss the application of ZDPs to ARHCER, consider the pose of the robot $\b q = (\b p, q)\in\mathcal{Q}$ where $\b p\in\R^3$ represents the global position in world frame and $q\in \mathbb{S}^3$ the robot's orientation quaternion. Taking the velocities to be $\b v = (\dot{\b p}, \b\omega) \in T_\b q\mathcal{Q}$ for $\dot{\b p}\in\R^3$ the global linear velocity and $\b \omega \in \mathfrak{s}^3$ the body frame angular rates, we can represent the full state as $\b x = (\b q, \b v)\in \mathcal{X}\triangleq T\mathcal{Q}$. 

ARCHER evolves under hybrid dynamics. As such, its flight and ground phase dynamics are governed by \eqref{eq:robot_dynamics} and it has two impact maps of the form \eqref{eqn:disc} (one for the ground to flight transition, and another for flight to ground). We treat the vertical hopping as an autonomous system, and we will focus our attention on how to stabilize the position of the robot via orientation. The flight dynamics can be decomposed into actuated states, i.e. the orientation coordinates, and unactuated states, i.e. position coordinates:
\begin{align*}
    \b \eta = \begin{bmatrix}
    q \\ \b \omega
\end{bmatrix}, \quad \b z = \begin{bmatrix}
    \b p \\ \b \dot{\b p}
\end{bmatrix}.
\end{align*}
Take $(\b \eta_k, \b z_k)$ to be a preimpact state. The ground phase does not depend on the control input, and the continuous-time evolution of the $\b z$ coordinates has an extremely weak dependence on the discrete-time control input $\b v_k$. 
We can assume $\b \Omega$ is independent from $\b v_k$ because the effect of different control inputs on impact time is negligible.

\subsection{Online Control Implementation}
Given a function $\b \psit$, the controller aims to stabilize its associated zeroing manifold $\mathcal{M}_{\b \psi}$. Consider a state $(\b \eta(t), \b z(t))$ during the flight phase. We set the desired orientation to $\b \eta_d(t) = \b \psit(\b z(t))$, and update this continuously throughout the flight phase. The desired set point is converted to a quaternion, $q_d$, which we stabilize using the following quaternion PD controller in the flight phase:
\begin{align*}
    \b u = -\b K_p\text{log}(q_d^{-1}q) - \b K_d \omega,
\end{align*}
for suitable gains $\b K_p, \b K_d$. This controller is applied at 1kHz. 

One key addition to the controller as compared to previous work \cite{csomay-shanklin_nonlinear_2023} is the application of flywheel spindown in the ground phase. When the robot is in contact with the floor, the following control action is applied:
\begin{align*}
    \b u = -\gamma\dot{\b \vartheta},
\end{align*}
where $\dot{\b \vartheta}\in \R^3$ represents the flywheel speed. This allows the system to maintain lower flywheel speeds and mitigates the problem of speed-torque constraints. This ground phase controller preserves the theoretical assumptions since the ground phase control is independent of output of the policy.

{\setlength{\textfloatsep}{-0pt}
\begin{figure}
\end{figure}
\begin{algorithm}[t!]
\caption{Monte Carlo Zero Dynamics Policy Training}
\label{algo:motGraph}
\begin{small}
\begin{algorithmic}[1]
\State \textbf{hyperparameters:} $(\Xi, \rho, \Upsilon)$
\State Number of MC samples, Learning Rate and Number of Steps
\State \textbf{Initialize} $\b \theta$ \Comment{Pretrained with reasonable policy}
\For{$i = 1:\Upsilon$}
\State $\b z \sim \text{UNIFORM}(\underline{\b z}, \overline{\b z})$
\State $\b \zeta_{\b \theta} \gets \begin{bmatrix}
    \b \psit(\b z) \\ \b z
\end{bmatrix}$
\State$\b x_0 \gets \b \Phi^{-1}(\b \zeta_{\b \theta})$
\State $\b x_{1:T}^*, \b v_{1:T}^* \gets \text{iLQR}(\b x_0)$
\State $\begin{bmatrix}
    \b\eta_{1}^*\left(\b \zeta_{\b \theta}(\b z)\right) \\
    \b z^*_{1}\left(\b \zeta_{\b \theta}(\b z)\right)
\end{bmatrix} \gets \b \Phi(\b x_1)$
\State $\b \theta_{i+1} \gets \b \theta_i - \rho \nabla_{\b \theta} \sum_{\b z}  \left\lVert \b\eta_{1}^*\left(\b \zeta_{\b \theta}(\b z)\right) - \b\psit\left(\b z^*_{1}\left(\b \zeta_{\b \theta}(\b z)\right)\right)\right\rVert^2_2 $
\EndFor
\State \Return $\b \theta$
\end{algorithmic}
\end{small}
 \label{alg:train}
\end{algorithm}
}

There are a few implementation differences from our theoretical implementation. The controller used in the proof of \cref{thm:stabilize_manifold} differs from ours by (1) predicting the preimpact state $\b z_{k+1}$, (2) tracking a trajectory $\b \eta_d(t)$ defined by a bezier polynomial, and (3), using a RES-CLF. Empirically, a well tuned PD controller was sufficient to stabilize the continuous time system, and the feedforward input tracking that a trajectory would provide was not necessary.

\begin{figure}[t!]
\vspace{2.5 mm}
    \centering
    \includegraphics[width=\columnwidth]{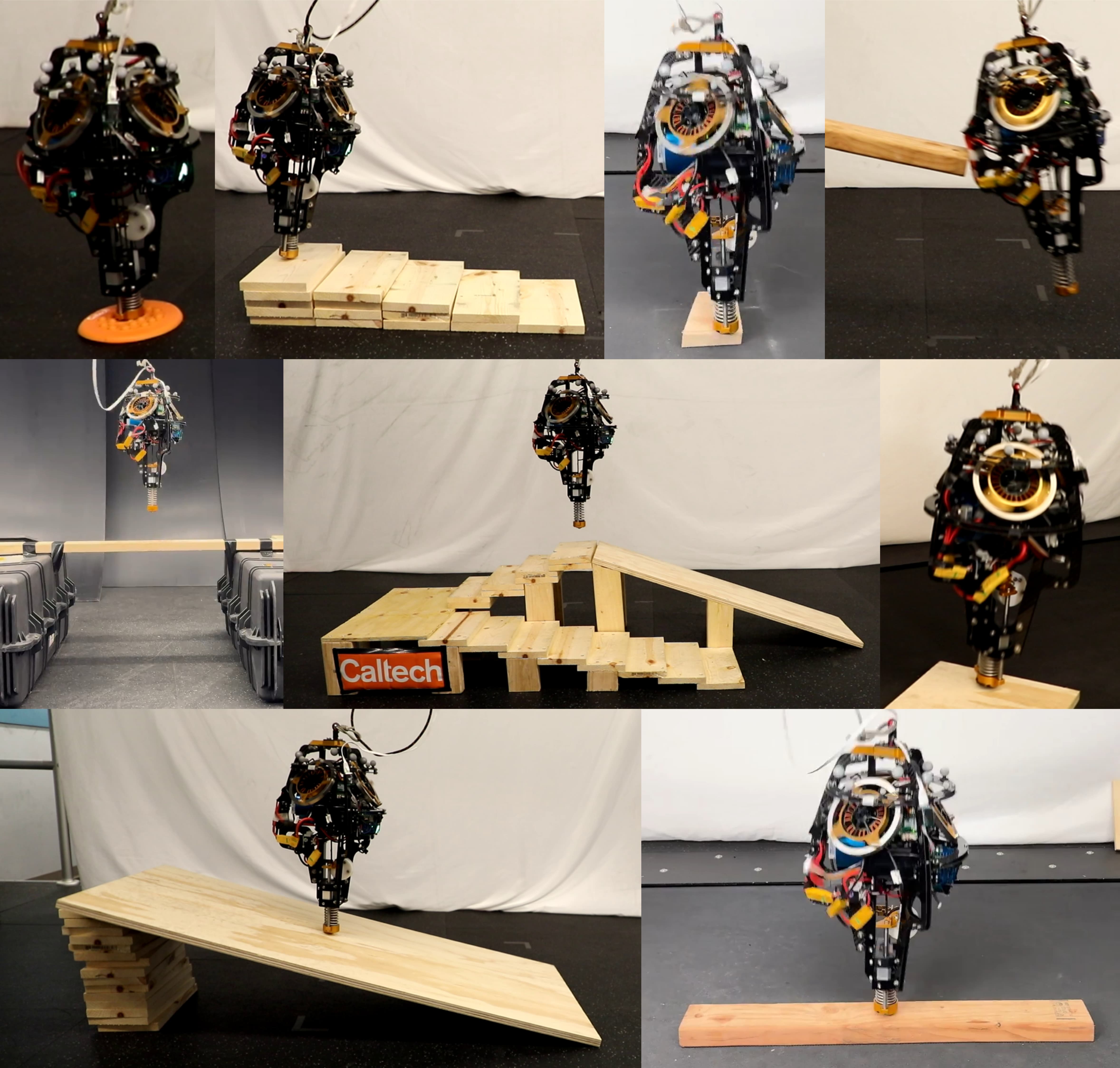}
    \caption{A snapshot of the experiments conducted with ARCHER, including set point tracking, disturbance rejection, and hopping over rough terrain.}
    \label{fig:experiment_collage}
    \vspace{-5.5 mm}
\end{figure}

\subsection{ZDP Optimization and Learning Details}

Notice that for discrete-time systems, \cref{eqn:ftocp} is a nonlinear program even if the value function is available. To solve this optimal control problem, we employ Iterative LQR (iLQR), subject to box input constraints \cite{tassa2014control}. The iLQR problem is solved in the $\b x$ variable, so the initial condition is obtain via $\b x = \b \Phi^{-1}(\b \eta, \b z)$. 
We implemented \cref{algo:motGraph} in the JAX \cite{jax2018github} and used a Network of 2 Layers with 256 hidden units each using ReLu activations. In our implementation of iLQR, we assume that the low-level controller has perfect tracking and exactly achieves the desired angle with zero angular velocity. This considerably simplifies the flight dynamics and therefore the trajectory optimization, allowing them to be solved for in closed form. The input bounds $\inputSetk$ were chosen such that the torque applied during flight is bounded by the difference between the post-impact state and the desired preimpact state.  We require gradients of the optimal control, $\frac{d\b v}{d \b x}$, as presented in \cite{amos2018differentiable} -- note that if no constraints are active, then this gradient is exactly the feedback matrix $\b K = \b Q_{\b v\b v}^{-1} \b Q_{\b v\b x}$ from the iLQR algorithm.

iLQR requires a stabilizing initial guess in order to converge; therefore, we use a Raibert heuristic for the first rollout. To eliminate this dependence, other optimal control methods could be used, for instance SQP. The authors experienced difficulty with the speed and accuracy of large-scale QP solvers in JAX and leveraged the fact that iLQR solves many small QPs for speed and stability. Additionally, for computational efficiency, we limit the number of iLQR iterations to five (empirically enough to obtain convergence for this system). The full code base for this project can be found at \cite{code}.

\section{Results and Limitations}
\subsection{Hardware Results}
A collection of the experiments conducted on ARCHER can be seen in Figure \ref{fig:experiment_collage}. The ARCHER hardware platform \cite{ambrose_creating_2022} consists of three KV115 T-Motors with 250 g flywheel masses attached for orientation control, and one U10-plus T-Motor attached to a 3-1 gear reduction to the foot via a cable and pulley system. The robot is powered by two 6 cell LiPo betteries connected in series, which can supply up to 50.8 V at over 100 A of current to the four ELMO Gold Solo Twitter motor controllers.
The policy $\b \psit$ was exported from JAX to an ONNX file, which is evaluated at 1kHz on an Ubuntu 20.04 machine with AMD Ryzen 5950x @ 3.4 GHz and 64 Gb RAM and torques are passed directly to the robot over ethernet. This controller does not require this amount of compute to run, and could be feasibly implemented on an NVIDIA Jetson or comparable board. A Kalman filter with projectile dynamics is used to filter the position estimates from optritrack in the flight phase. The manif library \cite{Deray-20-JOSS} is used to compute the $\log$ map for the quaternion PD controller.

We logged over 3,000 stable hops when deploying the ZDP method on the ARCHER hardware platform, a selection of which can be seen in \cref{fig:experiment_collage} and in the supplemental video \cite{noauthor_supplemental_video_2022}. \cref{fig:lqr_comp} depicts the desired impact angle, i.e. the learned policy evaluation, and the actual impact angle over the complete collection of all hardware tests. In general, as predicted by the theory, this manifold is both invariant under the feedback controller, and stable. Also interesting to note is that around the origin, the learned policy alignes with LQR, as presented in \cref{thm:exp_stab_value}. Notably, away from the origin, the learned policy diverges from LQR in order to maintain stability under the enforced input contstraints. A comparison between the trained policy and the application of a naive LQR controller when trying to track a setpoint 2 m away is seen in the left part of \cref{fig:lqr_comp}, wherein ZDPs maintain stability by implicitly enforcing discrete invariance and optimality over a horizon.

\begin{figure}[t!]
    \centering
        \vspace{5pt}
    \includegraphics[width=\columnwidth]{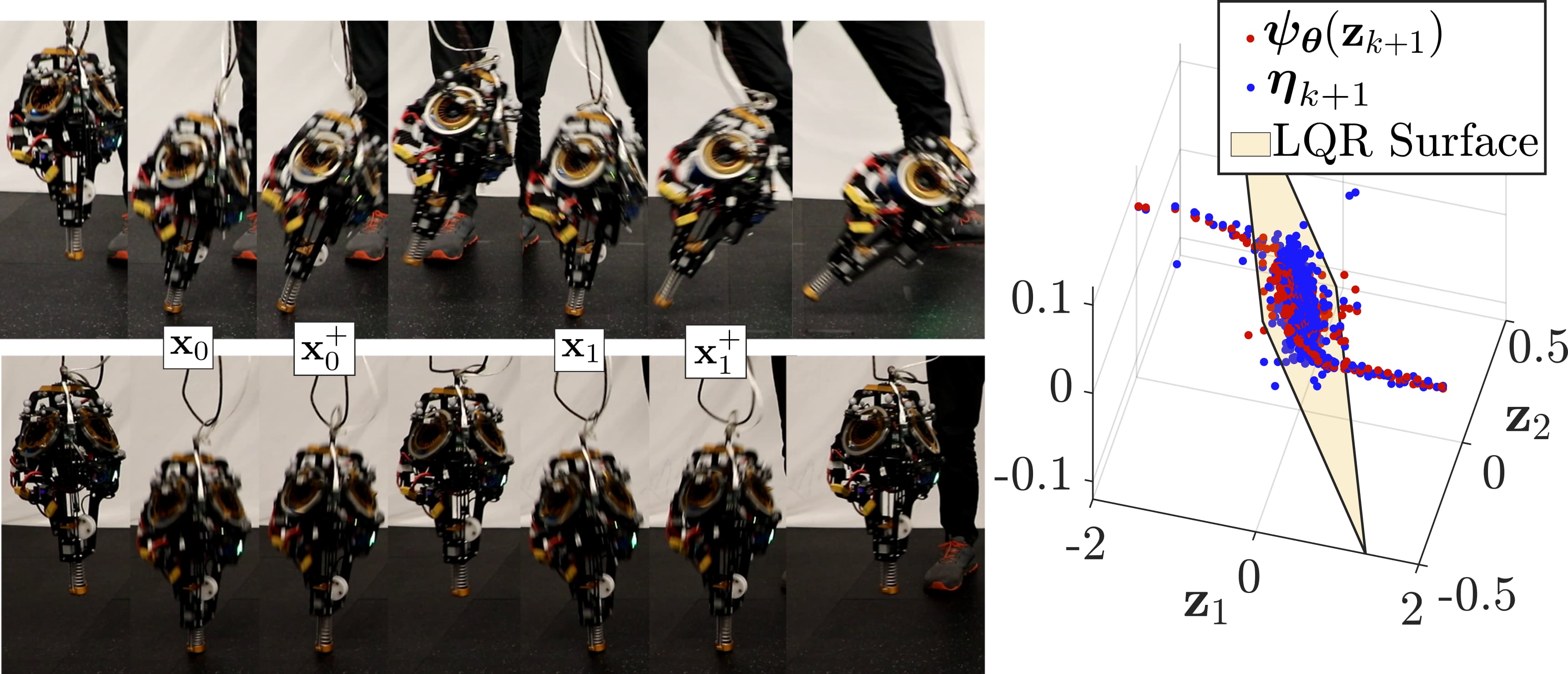}
    \caption{Left: A comparison between LQR (top) and ZDPs (bottom) while tracking a 2 m setpoint. Right: The output of the trained policy and the actual state at impact over 3000 hops, as compared to an LQR controller.}
    \label{fig:lqr_comp}
    \vspace{-5mm}
\end{figure}

The tight trajectory tracking and system behavior is seen in \cref{fig:square}, where ARCHER was asked to follow two laps of a 1 m square trajectory. As seen on the right of \cref{fig:square}, using a PD controller at the feedback level empirically resulted in the error (and therefore the torques) converging exponentially fast to a small neighborhood of zero during the flight phase. During this torque application, the flywheel speed can be seen to grow, while the ground phase controller is able to successfully regulate them close to zero. 

\begin{figure}[t!]
    \centering
        \vspace{5pt}
    \includegraphics[width=\columnwidth]{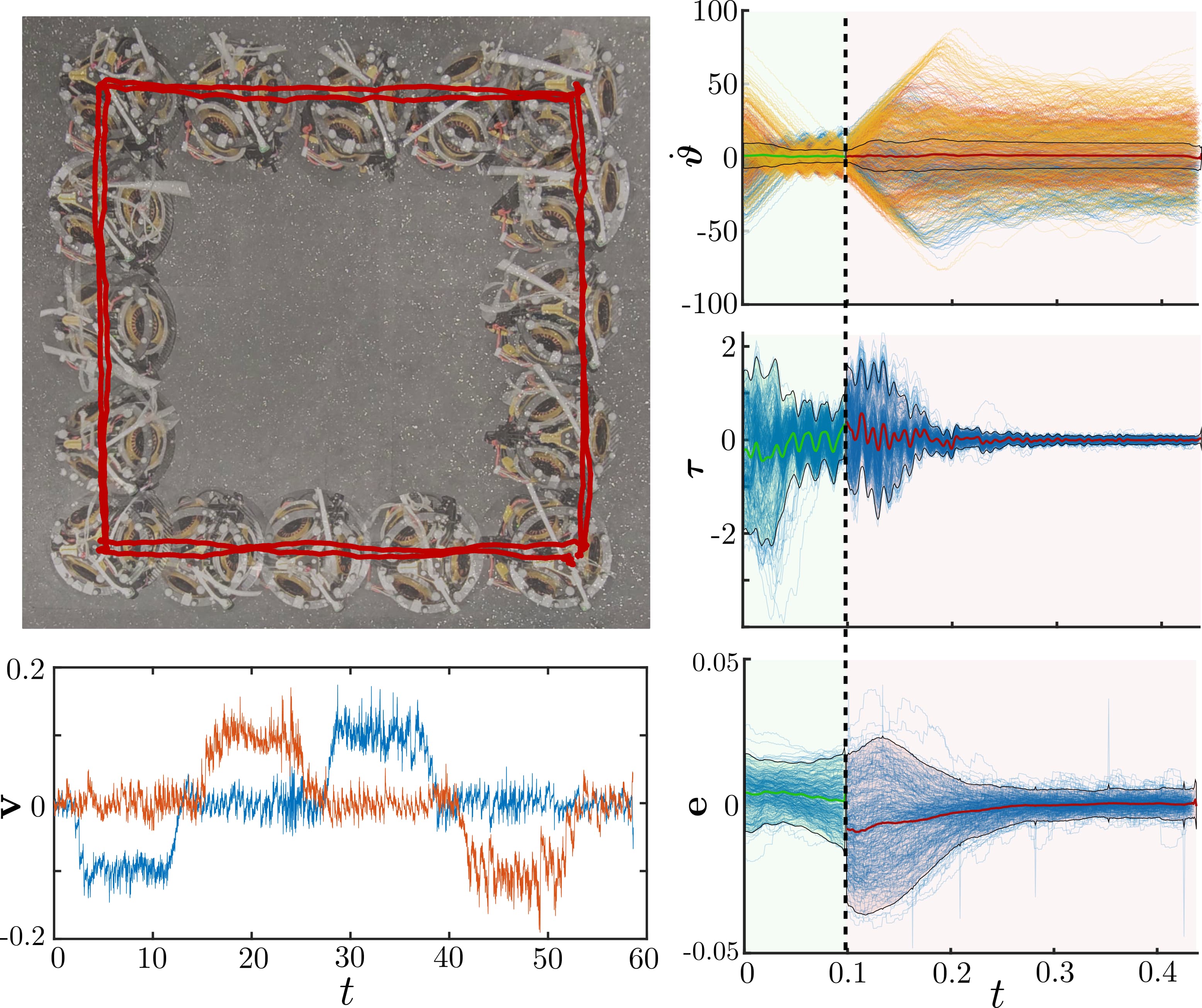}
    \caption{Square trajectory tracking. Left pane: overhead view with positional hardware data overlayed (top) and velocity tracking (bottom). Right pane: wheel velocities (top), torque (mid), and error (bottom) in the ground (green) and flight (red) phase with mean and 2$\sigma$ deviation.}
    \label{fig:square}
    \vspace{-4mm}
\end{figure}

\subsection{Limitations}
As training this policy involves querying the optimal control input and its gradients, each iteration of the training process is computationally expensive (2 seconds per iteration for a batch size of 30). The use of iLQR required a stabilizing controller to initialize the rollout and therefore can only do local improvements on a stabilizing policy. 
Furthermore, to avoid sampling initial conditions in the training pipeline which the hopper cannot stabilize, the policy $\b \psit$ was pretrained with a conservative Raibert heuristic. 

\section{Conclusion and Future Work}
We have proposed a method of synthesizing stabilizing feedback controllers for hybrid underactuated systems. By exploiting the zero dynamics decomposition, we demonstrated both theoretically and experimentally that stabilizing such systems can effectively be decomposed into designing a mapping which renders the discrete zeroing manifold invariant under optimal controllers and pairing it with a suitable tracking controller. Future work includes merging the proposed methods with RL controllers, applying to other legged systems, and developing a parallel theory for continuous time systems.

\section{Acknowledgements}
We would like to thank Murtaza Hathiyari for aiding with C++ code development and hardware experiment testing.

\bibliographystyle{IEEEtran}
\balance
\bibliography{main}

\begin{thebibliography}{10}
\providecommand{\url}[1]{#1}
\csname url@samestyle\endcsname
\providecommand{\newblock}{\relax}
\providecommand{\bibinfo}[2]{#2}
\providecommand{\BIBentrySTDinterwordspacing}{\spaceskip=0pt\relax}
\providecommand{\BIBentryALTinterwordstretchfactor}{4}
\providecommand{\BIBentryALTinterwordspacing}{\spaceskip=\fontdimen2\font plus
\BIBentryALTinterwordstretchfactor\fontdimen3\font minus \fontdimen4\font\relax}
\providecommand{\BIBforeignlanguage}[2]{{%
\expandafter\ifx\csname l@#1\endcsname\relax
\typeout{** WARNING: IEEEtran.bst: No hyphenation pattern has been}%
\typeout{** loaded for the language `#1'. Using the pattern for}%
\typeout{** the default language instead.}%
\else
\language=\csname l@#1\endcsname
\fi
#2}}
\providecommand{\BIBdecl}{\relax}
\BIBdecl

\bibitem{sastry_linearization_1999}
S.~Sastry, ``\BIBforeignlanguage{en}{Linearization by {State} {Feedback}},'' in \emph{\BIBforeignlanguage{en}{Nonlinear {Systems}: {Analysis}, {Stability}, and {Control}}}, ser. Interdisciplinary {Applied} {Mathematics}, S.~Sastry, Ed.\hskip 1em plus 0.5em minus 0.4em\relax Springer, 1999, pp. 384--448.

\bibitem{pmlr-v168-rodriguez22a}
I.~D.~J. Rodriguez, N.~Csomay-Shanklin, Y.~Yue, and A.~D. Ames, ``Neural gaits: Learning bipedal locomotion via control barrier functions and zero dynamics policies,'' in \emph{Proceedings of The 4th Annual L4DC}, vol. 168.\hskip 1em plus 0.5em minus 0.4em\relax PMLR, Jun 2022, pp. 1060--1072.

\bibitem{compton2024constructivenonlinearcontrolunderactuated}
W.~Compton, I.~D.~J. Rodriguez, N.~Csomay-Shanklin, Y.~Yue, and A.~D. Ames, ``Constructive nonlinear control of underactuated systems via zero dynamics policies,'' \emph{preprint arXiv:2408.14749}, 2024.

\bibitem{liberzon_calculus_2012}
D.~Liberzon, \emph{Calculus of {Variations} and {Optimal} {Control} {Theory}: {A} {Concise} {Introduction}}.\hskip 1em plus 0.5em minus 0.4em\relax Princeton University Press, 2012.

\bibitem{borrelli2017predictive}
F.~Borrelli, A.~Bemporad, and M.~Morari, \emph{Predictive control for linear and hybrid systems}.\hskip 1em plus 0.5em minus 0.4em\relax Cambridge University Press, 2017.

\bibitem{mayne2000constrained}
D.~Mayne, J.~Rawlings, C.~Rao, and P.~Scokaert, ``Constrained model predictive control: Stability and optimality,'' \emph{Automatica}, vol.~36, no.~6, pp. 789--814, 2000.

\bibitem{wensing2022optimizationbasedcontroldynamiclegged}
P.~M. Wensing, M.~Posa, Y.~Hu, A.~Escande, N.~Mansard, and A.~D. Prete, ``Optimization-based control for dynamic legged robots,'' \emph{Trans. Rob.}, vol.~40, p. 43–63, oct 2023.

\bibitem{khazoom2024humanoid}
C.~Khazoom, S.~Hong, M.~Chignoli, E.~Stanger-Jones, and S.~Kim, ``Tailoring solution accuracy for fast whole-body model predictive control of legged robots,'' \emph{preprint arXiv:2407.10789}, 2024.

\bibitem{li2024cafempccascadedfidelitymodelpredictive}
H.~Li and P.~M. Wensing, ``Cafe-mpc: A cascaded-fidelity model predictive control framework with tuning-free whole-body control,'' \emph{preprint arXiv:2403.03995}, 2024.

\bibitem{westervelt_hybrid_2003}
E.~Westervelt, J.~Grizzle, and D.~Koditschek, ``Hybrid zero dynamics of planar biped walkers,'' \emph{IEEE Transactions on Automatic Control}, vol.~48, no.~1, pp. 42--56, Jan. 2003.

\bibitem{reher2021dynamic}
J.~Reher, ``Dynamic bipedal locomotion: From hybrid zero dynamics to control lyapunov functions via experimentally realizable methods,'' Ph.D. dissertation, California Institute of Technology, 2021.

\bibitem{Schulmanetal_ICLR2016}
J.~Schulman, P.~Moritz, S.~Levine, M.~Jordan, and P.~Abbeel, ``High-dimensional continuous control using generalized advantage estimation,'' in \emph{Proceedings of ICLR}, 2016.

\bibitem{miki2022learning}
T.~Miki, J.~Lee, J.~Hwangbo, L.~Wellhausen, V.~Koltun, and M.~Hutter, ``Learning robust perceptive locomotion for quadrupedal robots in the wild,'' \emph{Science Robotics}, vol.~7, no.~62, p. eabk2822, 2022.

\bibitem{li2024reinforcementlearningversatiledynamic}
Z.~Li, X.~B. Peng, P.~Abbeel, S.~Levine, G.~Berseth, and K.~Sreenath, ``Reinforcement learning for versatile, dynamic, and robust bipedal locomotion control,'' \emph{preprint arXiv:2401.16889}, 2024.

\bibitem{suh2022differentiable}
H.~J. Suh, M.~Simchowitz, K.~Zhang, and R.~Tedrake, ``Do differentiable simulators give better policy gradients?'' in \emph{ICML}.\hskip 1em plus 0.5em minus 0.4em\relax PMLR, 2022, pp. 20\,668--20\,696.

\bibitem{raibert_experiments_1984}
M.~H. Raibert, H.~B. Brown, and M.~Chepponis, ``\BIBforeignlanguage{en}{Experiments in {Balance} with a {3D} {One}-{Legged} {Hopping} {Machine}},'' \emph{\BIBforeignlanguage{en}{IJRR}}, vol.~3, no.~2, pp. 75--92, Jun. 1984, publisher: SAGE Publications Ltd STM.

\bibitem{kajita20013d}
S.~Kajita, F.~Kanehiro, K.~Kaneko, K.~Yokoi, and H.~Hirukawa, ``The 3d linear inverted pendulum mode: A simple modeling for a biped walking pattern generation,'' in \emph{Proceedings 2001 IEEE/RSJ ICIRS (Cat. No. 01CH37180)}, vol.~1.\hskip 1em plus 0.5em minus 0.4em\relax IEEE, 2001, pp. 239--246.

\bibitem{han20223d}
B.~Han, H.~Yi, Z.~Xu, X.~Yang, and X.~Luo, ``3d-slip model based dynamic stability strategy for legged robots with impact disturbance rejection,'' \emph{Scientific Reports}, vol.~12, no.~1, p. 5892, 2022.

\bibitem{isidori_elementary_1995}
A.~Isidori, ``\BIBforeignlanguage{en}{Elementary {Theory} of {Nonlinear} {Feedback} for {Single}-{Input} {Single}-{Output} {Systems}},'' in \emph{\BIBforeignlanguage{en}{Nonlinear {Control} {Systems}}}, ser. Communications and {Control} {Engineering}.\hskip 1em plus 0.5em minus 0.4em\relax London: Springer, 1995, pp. 137--217.

\bibitem{westervelt2003hybrid}
E.~R. Westervelt, J.~W. Grizzle, and D.~E. Koditschek, ``Hybrid zero dynamics of planar biped walkers,'' \emph{IEEE Transactions on Automatic Control}, vol.~48, no.~1, pp. 42--56, 2003.

\bibitem{reher2021control}
J.~Reher and A.~D. Ames, ``Control lyapunov functions for compliant hybrid zero dynamic walking,'' \emph{preprint arXiv:2107.04241}, 2021.

\bibitem{da2017combiningtrajectoryoptimizationsupervised}
X.~Da and J.~Grizzle, ``Combining trajectory optimization, supervised machine learning, and model structure for mitigating the curse of dimensionality in the control of bipedal robots,'' \emph{The International Journal of Robotics Research}, vol.~38, no.~9, pp. 1063--1097, 2019.

\bibitem{noauthor_supplemental_video_2022}
\BIBentryALTinterwordspacing
``Supplemental video.'' [Online]. Available: \url{{https://vimeo.com/923800815}}
\BIBentrySTDinterwordspacing

\bibitem{ames2017hybrid}
A.~D. Ames and I.~Poulakakis, ``Hybrid zero dynamics control of legged robots,'' \emph{Bioinspired Legged Locomotion: Models, Concepts, Control and Applications}, pp. 292--331, 2017.

\bibitem{csomay2022multi}
N.~Csomay-Shanklin, A.~J. Taylor, U.~Rosolia, and A.~D. Ames, ``Multi-rate planning and control of uncertain nonlinear systems: Model predictive control and control lyapunov functions,'' in \emph{2022 IEEE 61st CDC}.\hskip 1em plus 0.5em minus 0.4em\relax IEEE, 2022, pp. 3732--3739.

\bibitem{csomay-shanklin_nonlinear_2023}
N.~Csomay-Shanklin, V.~D. Dorobantu, and A.~D. Ames, ``\BIBforeignlanguage{en}{Nonlinear {Model} {Predictive} {Control} of a {3D} {Hopping} {Robot}: {Leveraging} {Lie} {Group} {Integrators} for {Dynamically} {Stable} {Behaviors}},'' in \emph{\BIBforeignlanguage{en}{2023 {ICRA}}}.\hskip 1em plus 0.5em minus 0.4em\relax London, United Kingdom: IEEE, May 2023, pp. 12\,106--12\,112.

\bibitem{tassa2014control}
Y.~Tassa, N.~Mansard, and E.~Todorov, ``Control-limited differential dynamic programming,'' in \emph{2014 ICRA}.\hskip 1em plus 0.5em minus 0.4em\relax IEEE, 2014, pp. 1168--1175.

\bibitem{jax2018github}
J.~Bradbury, R.~Frostig, P.~Hawkins, M.~J. Johnson, C.~Leary, D.~Maclaurin, G.~Necula, A.~Paszke, J.~Vander{P}las, S.~Wanderman-{M}ilne, and Q.~Zhang, ``{JAX}: composable transformations of {P}ython+{N}um{P}y programs,'' 2018.

\bibitem{amos2018differentiable}
B.~Amos, I.~Jimenez, J.~Sacks, B.~Boots, and J.~Z. Kolter, ``Differentiable mpc for end-to-end planning and control,'' \emph{Advances in neural information processing systems}, vol.~31, 2018.

\bibitem{code}
\BIBentryALTinterwordspacing
``Code,'' 2024. [Online]. Available: \url{{https://github.com/ivandariojr/LearnedZeroDynamicsPolicies}}
\BIBentrySTDinterwordspacing

\bibitem{ambrose_creating_2022}
E.~R. Ambrose, ``\BIBforeignlanguage{en}{Creating {ARCHER}: {A} {3D} {Hopping} {Robot} with {Flywheels} for {Attitude} {Control}},'' Ph.D. dissertation, California Institute of Technology, 2022.

\bibitem{Deray-20-JOSS}
J.~Deray and J.~Solà, ``Manif: A micro {L}ie theory library for state estimation in robotics applications,'' \emph{Journal of Open Source Software}, vol.~5, no.~46, p. 1371, 2020.

\end{thebibliography}
\end{document}